\documentclass[twoside,11pt,fleqn]{article}
\usepackage{jair, theapa, rawfonts}

\usepackage{lineno}
\usepackage{amssymb,amsmath,amsthm}
\usepackage{enumerate}
\usepackage{xspace}
\usepackage{graphicx}
\usepackage{tikz}
\usetikzlibrary{arrows,shapes}
\usepackage{url}

\usepackage[draft]{fixme}

\makeatletter
\newcommand{\@abbrev}[3]{
  \def\c@a@def##1{
      \if ##1.
        \relax
      \else
        \@ifdefinable{\@nameuse{#1##1}}{\@namedef{#1##1}{#2##1}}
        \expandafter\c@a@def
      \fi
    }
  \c@a@def #3.
}
\@abbrev{bb}{\mathbb}{ABCDEFGHIJKLMNOPQRSTUVWXYZ}
\@abbrev{bf}{\mathbf}{ABCDEFGHIJKLMNOPQRSTUVWXYZabcdefghijklmnopqrstuvwxyz}
\@abbrev{bit}{\boldsymbol}{ABCDEFGHIJKLMNOPQRSTUVWXYZabcdefghijklmnopqrstuvwxyz}
\@abbrev{cal}{\mathcal}{ABCDEFGHIJKLMNOPQRSTUVWXYZ}
\@abbrev{frak}{\mathfrak}{ABCDEFGHIJKLMNOPQRSTUVWXYZabcdefghijklmnopqrstuvwxyz}
\@abbrev{rm}{\mathrm}{ABCDEFGHIJKLMNOPQRSTUVWXYZabcdefghijklmnopqrstuvwxyz}
\@abbrev{scr}{\mathscr}{ABCDEFGHIJKLMNOPQRSTUVWXYZ}
\@abbrev{sf}{\mathsf}{ABCDEFGHIJKLMNOPQRSTUVWXYZabcdefghijklmnopqrstuvwxyz}
\makeatother

\newcommand{\SO}{\text{SO}}

\newcommand{\TC}{\text{TC}}
\newcommand{\LFP}{\text{LFP}}
\newcommand{\ol}{\overline}
\newcommand{\suc}{\ensuremath{\text{SUC}}}
\newcommand{\nsuc}{\ensuremath{\text{SUC}_2}}
\newcommand{\exactlyone}{\ensuremath{\text{exactly\_one}}}

\renewcommand{\phi}{\varphi}
\newcommand{\bigland}{\bigwedge}
\newcommand{\biglor}{\bigvee}

\newcommand{\liff}{\leftrightarrow}
\newcommand{\dom}{\text{dom}}
\newcommand{\sgn}{\text{sgn}}
\newcommand{\R}{\ensuremath{\bbR}\xspace}

\newtheorem{theorem}{Theorem}
\newtheorem{definition}{Definition}
\newtheorem{example}{Example}
\newtheorem{lemma}{Lemma}

\newcommand\depqbf{\textsc{depqbf}\xspace}

\newcommand\lparse{\textsc{lparse}\xspace}
\newcommand\gringo{\textsc{gringo}\xspace}

\newcommand\degms{\textsc{de-gms}\xspace}

\newcommand\rareqs{\textsc{rareqs}\xspace}

\newcommand{\struc}{\mbox{\sf Struc}}
\newcommand{\instt}{\mbox{\sf inst}}

\newcommand{\skipack}{\thanks{Supported in part by Kakenhi Grant No.\ 15H00847, `Exploring the Limits of Computation (ELC)'.} }
\begin{document}

\title{Machine Learning with Guarantees using Descriptive~Complexity and SMT Solvers}
\author{%
\name Charles Jordan\skipack \email skip@ist.hokudai.ac.jp \\ 
\addr Hokkaido University \AND
\name \L{}ukasz Kaiser \email lukaszkaiser@gmail.com \\
\addr Google Brain \& CNRS (LIAFA)
}

\maketitle

\begin{abstract}
Machine learning is a thriving part of computer science.
There are many efficient approaches to machine learning
that do not provide strong theoretical guarantees,
and a beautiful general learning theory.
Unfortunately, machine learning approaches that give strong theoretical
guarantees have not been efficient enough to be applicable.

In this paper we introduce a logical approach to machine learning.
Models are represented by tuples of logical formulas and inputs and outputs
are logical structures. We present our framework together with
several applications where we evaluate it using SAT and SMT solvers.
We argue that this approach to machine learning is particularly suited
to bridge the gap between efficiency and theoretical soundness.

We exploit results from descriptive complexity theory to prove
strong theoretical guarantees for our approach.
To show its applicability, we present experimental results including
learning complexity-theoretic reductions 
rules for board games. We also explain how neural networks fit
into our framework, although the current implementation does not scale
to provide guarantees for real-world neural networks.
\end{abstract}



\section{Introduction}
\label{sec:intro}

Machine learning has a long history in computer science.
It includes techniques like neural networks, Bayesian models,
genetic programming, inductive synthesis and many others.
In various applications such as voice recognition, these methods  are now used by over a billion people.
Some machine learning methods give some guarantees of success,
but they are usually dependent on a number of strong assumptions
about the distribution of inputs or the existence of a model
with a particular form. One would hope for much stronger theoretical
guarantees of success, but no widely used machine learning methods
provide them, so it is difficult to know when they will work.

One key question faced by any machine learning system is:
what kind of models will it generate? In neural networks
one asks whether the architecture is feed-forward or recurrent
and how many layers it has. In genetic programming one asks
for the program representation and which functions are built-in.
In general, each machine learning system must make this choice.
If the class of generated models or programs is too broad,
it might be impossible to learn them efficiently.
If it is too narrow, it might not suffice for the task at hand.
To solve a different task, one might need a different kind of model.
But is there an efficient systematic way to know the kind of model needed for a task?

There exists a number of reasonably broad machine learning methods.
So one efficient way to apply machine learning to a new task is
to try all of these techniques in turn. But what if they all fail?
Since there are few theoretical guarantees, it is seldom clear whether
the reason for the failure is the model, wrong parameters or simply a bug.
On the other hand, there are systematic ways to explore the space
of all models that come with strong guarantees. One can, for example,
just enumerate all programs in a programming language of choice.
Of course, this is too inefficient for any practical purpose.

We propose an efficient way to systematically explore the space
of all models in given computational complexity class. 
It is based on findings from \emph{descriptive complexity} where one studies how
programs in different complexity classes, such as NL, P, or NP,
can be characterized syntactically. Recent work~\cite{CIM10,IGIS10,JK13}
suggests that \emph{logical queries} of various logics are a particularly
good choice for such syntactic representations. Using logical queries
allows us, on the one hand, to exploit results from descriptive complexity
to get theoretical guarantees for our algorithms. On the other hand, it
allows us to leverage recent advances in SAT, QBF and other SMT solvers
to address more practical concerns.  In general, learning correct models
satisfying a given condition is not computable.  We first
introduce some restrictions on the size of structures and models
that we look for and present an optimized algorithm for that restricted
problem. Then, we show how to iterate these solutions to get reasonably
efficient semi-decision procedures with strong theoretical guarantees.

One advantage of a machine learning approach with such strong guarantees
is that it can give negative answers to certain questions.
For example, our approach can sometimes prove that there is no model in a certain
class that is sufficient for the given machine learning task.
A main disadvantage is that this does not scale to large models.
Still, it can be used to enrich our understanding
of complexity theory and model classes, even when only applied on a small scale.
For example, in Section~\ref{subsec:reduc} we show how our approach can be
used to prove that certain reductions between complexity classes do not exist.

The rest of the paper is organized as follows. After introducing
related work, we review the necessary background from logic
in Section~\ref{sec:logic}. We define our learning model and
prove the main theorems in Section~\ref{sec:learn}. 
In particular, Theorem~\ref{thm:guidance} shows that if a model
satisfying the specification exists and its computational
complexity is in a given complexity class (e.g., NL, P, NP),
then our approach is guaranteed to find it. Knowing
the theoretical guarantees, we devote Section~\ref{sec:appl}
to applications and experiments. While our technique cannot reach the scale of
machine learning systems that come with no guarantees, we show
results on learning a number of non-trivial tasks that require
very different kinds of models. We also show how other machine
learning techniques, such as neural networks, fit into our approach.

\subsection{Related Work}
Machine learning and inductive synthesis (the subfield closest to our approach) have long histories; there is a tremendous amount of work that we do not cover.
We refer the reader to \cite{G10surv,K10surv}
for a general perspective on inductive synthesis.

Much of our motivation comes from recent papers using ideas from descriptive
complexity in inductive synthesis.  For example, given a specification in
an expressive logic (second-order), \cite{IGIS10} synthesized equivalent
formulas in less expressive logics which can be evaluated more efficiently.
Automatically finding complexity-theoretic reductions between computational
problems was first considered by \cite{CIM10}.  They focused on quantifier-free
reductions, a weak class of reduction defined by tuples of quantifier-free
formulas.

Both problems are essentially the same -- finding formulas in a 
particular form that satisfy desired properties.  However, the
implementations are separate and not publicly available.  
\cite{JK13} compared a number of different approaches to
reduction finding.  In this paper, we introduce a more general approach --
allowing the user to specify an outline of the desired formula and
a specification that it must satisfy.  We provide a freely available
implementation that can be used to experiment with
various synthesis problems.

Another source of motivation for this paper comes from recent 
successes~\cite{HKM16,KL15} using logic solvers to resolve interesting
problems in mathematics.  In a similar fashion, determining the existence of
formulas can resolve open questions and our approach gives a way to leverage
modern solvers in new areas.  In particular, recent progress on sequential
and parallel QBF solvers suggests that this may be a promising approach to
certain problems that do not have compact SAT encodings (assuming 
NP$\neq$PSPACE).

\section{Background in Logic and Descriptive Complexity}
\label{sec:logic}

In this section we briefly review the necessary background from
descriptive complexity.  For more details, see~\cite{I99} or
Chapter~3 of~\cite{FMT07} for an overview and background, or~\cite{GM96}
and~\cite{GGmf} for details on \R-structures and their logics.

There are many possible representations of models or programs; in this paper, we
focus on \emph{logical} representations.  One benefit of the logical approach
is that we are able to treat structures such as graphs directly,
instead of encoding them into words or numbers.
This allows us to express many interesting models succinctly.
Additionally, formulas have natural normal forms.  These provide guidance
for hypothesis spaces, and improve understandability of learned models.
Also, it turns out that searching for logical formulas can be
translated to inputs for SAT and SMT solvers in a natural way.

Here, we consider models (programs) that transform given inputs into outputs and we
represent these inputs and outputs as logical structures (for example,
graphs or binary strings).  Although graphs are the most
common and suffice for many examples, they do not provide access to
computation with real numbers.  Metafinite structures~\cite{GGmf} are an
extension of relational structures that was introduced to resolve issues of
this kind.  We use \emph{\R-structures}~\cite{GM96}, a particular kind
of metafinite structures that provides limited access to computation with
real numbers (of course, relational structures are available as a special case).
This choice is not arbitrary -- there are many deep connections between logics
and complexity classes, and the relevant logics for \R-structures maintain these
connections (see below or~\cite{GM96}).

An \emph{\R-signature} is a tuple of predicate symbols $R_i$ with
arities~$a_i$, constant symbols $c_j$, and function symbols~$f_i$ with
arities~$b_i$:
\[\tau:=(R_1^{a_1},\ldots,R_r^{a_r},c_1,\ldots,c_s,f_1^{b_1},\ldots,f_t^{b_t})\,.\]
A \emph{$\tau$-structure} $A$ consists of a finite set $U$, called the \emph{universe},
an $a_i$-ary relation over $U$ for each predicate
symbol of $\tau$, a definition -- an element of $U$ -- for each constant symbol, 
and a definition mapping $U^{b_i}$ to \R for each function symbol:
\[(U, R_1\subseteq U^{a_1},\ldots,R_r\subseteq U^{a_r},
  c_1\in U,\ldots,c_s\in U,
  f_1\colon U^{b_1}\hspace{-0.5pt}\rightarrow\mathbb{R},\ldots,f_t\colon   U^{b_t}\hspace{-0.5pt}\rightarrow\mathbb{R}).\]
We always set $n=|U|$ and identify $U$ with the natural numbers $\{0,\ldots,n-1\}$.
Signatures containing no function symbols are called \emph{relational signatures}
and the corresponding structures are called \emph{relational structures}.

Many models make use of function symbols on the finite part $U$ of the structure,
i.e., functions $g\colon U^{a_g}\rightarrow U$.  We can represent these using a predicate
for the characteristic function.  One can, in a similar fashion, represent predicates with
function symbols for the characteristic function or represent constants with monadic
predicates.  We use the above definitions for simplicity.

\begin{example}
The (relational) signature for directed graphs contains a single,
binary predicate symbol $E$ and so a directed graph consists of a
finite set $U$ of vertices and a binary edge relation.  These graphs
may contain loops. The \R-signature
for a complete weighted directed graph contains one binary function
symbol $e$ which maps pairs from $U$ to their weights.
\end{example}

We denote the set of all $\tau$-structures by $\struc(\tau)$ and the set of $\tau$-structures with
universe size $n$ as $\struc^n(\tau)$. We also use the notion of an \emph{$\bbR$-modification} of
a structure $A$. We say that $B$ is an $\bbR$-modification of a structure $A$, written $B \sim_\bbR A$,
if it has the same signature, universe, constants and relations -- but it may differ on the values of its real-valued functions.

Our models are built from formulas in various logics.
Formulas of \emph{first-order logic} over a signature $\tau$ are built in the following way.
First, we fix a countable set of first-order variables $x_i$ -- these range only over
the finite part $U$ of structures.  Then, we fix an explicit enumeration of the set
of algebraic real numbers and denote these constants $r_i$.
Using those, we define the set of atomic \emph{number terms} $t$ and
\emph{first-order formulas} $\phi$ by the following BNF grammar.
\begin{align*}
t \ := & \ \ r_i \ | \ f_i(x_1, \ldots, x_{b_i}) \ | \ t + t \ | \ t -  t \ | \ t \cdot t \ | \ t/t \ | \ \sgn(t)
  \ | \sum_x t \ | \chi(\phi) \\
\phi \ := & \  \ x_i=x_j \ |  \ x_i=c_j \ | \ x_i<x_j \ | \ x_i<c_j \ | 
    \ R_i(x_1, \ldots, x_{a_i}) \ | \ \neg \phi \ |  \\
& \ \ \phi \lor \phi \ | \ \phi \land \phi \ | \ 
      \exists x_i \, \phi \ | \ \forall x_i \, \phi \ |\ 
       t = t \ |\ t < t
\end{align*}
where $x_i$ are first-order variables and $t$ are number terms.
The semantics, given an assignment of the variables $x_i$ to
elements $e_i$ of the structure, is defined in the natural way.
We interpret $\sgn(a)$ as the \emph{sign} of the real number $a$, i.e. 
\[\sgn(a)=\begin{cases}
  -1 & \mbox{if }a<0 \\
  0 & \mbox{if }a=0\\
  1 & \mbox{if }a>0.
  \end{cases}\]
 The sum $\sum_x t(x)$ is computed in the natural way: we compute $t(x)$ for each
 assignment of $x$ and add them. The term $\chi(\phi)$ stands for the characteristic
 function of the formula $\phi$, i.e., it is $1$ if $\phi$ holds and $0$ otherwise.
 We will often use the following abbreviation:
 \[ \sum_{x_1,\dots,x_n \, :\, \phi} t \ := \
    \sum_{x_1}\sum_{x_2}\dots\sum_{x_n} \chi(\phi) \cdot t\,.\]
 Note that the quantifiers are restricted to the finite part $U$ of structures, and do not range over $\mathbb{R}$.
 We use FO to refer to first-order logic on relational structures and FO$_\R$ to refer to first-order logic on \R-structures.
  
  \begin{example}
  Consider the signature of weighted graphs, 
  $\tau_w:=(e^2)$.  The first-order formula
  \begin{align*}
    \forall x,y,z &\left(x\neq y\land x\neq z\land y\neq z\right)\quad\rightarrow\quad
     \left( e(x,z) \le e(x,y)+e(y,z) \right)
  \end{align*}
  holds exactly if the triangle inequality is satisfied by
  all triangles in the graph. Note that we use ``$\le$'' and
   ``$\rightarrow$''; formally these are abbreviations that can be
    rewritten according to our definition of FO.
  \end{example}



In descriptive complexity it is very common to add additional numeric predicates to structures.
Here, we use $\suc(x,y)$ to mean $y=x+1$ and insist that structures
define this faithfully.  Note that this can be defined using the ordering, equality and a first-order quantifier. 
However, we often consider fragments of first-order logic where quantifiers are restricted or not available,
and having $\suc(x,y)$ can be important in such situations.

\subsection{Queries}

Single formulas can be used to define properties or decision problems, but in 
general we represent models as queries (also called \emph{interpretations}).  
Queries map $\sigma$-structures to $\tau$-structures, defining the universe, 
relations, constants and functions using
logical formulas.  A first-order query from $\sigma$-structures to
$\tau$-structures is an $r+s+t+2$-tuple,
\[q\ :=\ (k,\varphi_0,\varphi_1,\ldots,\varphi_r,\psi_1,\ldots,\psi_s,\delta_1,\ldots,\delta_t)\,.\]
The number $k\in\bbN$ is the \emph{dimension} of the query.
Each $\varphi_i$, $\psi_j$ is a first-order formula over the signature~$\sigma$.
Let $A$ be a $\sigma$-structure with universe $U^A$.
The formula~$\varphi_0$ has free variables $x_1,\ldots,x_k$ and defines
the universe $U$ of $q(A)$, \[U\ :=\ \left\{(u_1,\ldots,u_k) \mid
 u_i\in U^A, A\models\varphi_0(u_1,\ldots,u_k)\right\}.\]
That is, the new universe consists of $k$-tuples of elements of the
old universe, where $\varphi_0$ determines which $k$-tuples are included.

Each remaining $\varphi_i$ has free variables
$x_1^1,\ldots,x_1^{k},x_2^1,\ldots,x_{a_i}^k$ and defines
\[R_i := \left\{ \left((u_1^1,\ldots,u_1^k),\ldots,(u_{a_i}^1,\ldots,u_{a_i}^k)\right) \mid
A\models\varphi_i(u_1^1,\ldots,u_{a_i}^k) \right\}\cap U^{a_i}\,.\]
That is, $\varphi_i$ determines which of the $a_i$-tuples of $U$ are included
in $R_i$.  Next, each $\psi_i$ has free variables $x_1,\ldots,x_k$
and defines $c_i$ as the lexicographically minimal $(u_1,\ldots,u_k)\in U$
such that $A\models\psi_i(u_1,\ldots,u_k)$.
Finally, each $\delta_i$ is a number term that has free variables
$x_1^1,\ldots,x_1^k,x_2^1,\ldots,x_{b_i}^k$.  It defines
\[f_i\left((u_1^1,\ldots,u_1^k),\ldots,(u_{b_i}^1,\ldots,u_{b_i}^k)\right):=\delta_i(u_1^1,\ldots,u_{b_i}^k)\,.\]

First-order queries therefore transform $\sigma$-structures into
$\tau$-structures, and we write $q(A)$ to represent the
resulting $\tau$-structure.  The restriction to first-order logic here
is not essential -- given a logic $\calL$, we define $\calL$-queries in
an analogous way.

\begin{example}
Consider the (relational) vocabularies $\tau_S:=(S^1)$ and $\tau_G:=(E^2)$.  We interpret $\tau_S$-structures 
as binary strings where bit $i$ is $1$ if $S(i)$, and $\tau_G$-structures as graphs.  The following
first-order query gives a simple transformation from graphs to binary strings:
\[q_A:=\left(2,\top,E(x_1^1,x_1^2)\right)\,.\]
Given a graph with vertices $U=\{0,\ldots,n-1\}$, this query produces a binary string with bit positions
labeled by pairs $(i,j)\in U^2$. A bit $(i,j)$ is $1$ if $E(i,j)$.  Given that we always identify universes
with subsets of the naturals, we re-label these pairs lexicographically and the resulting string is essentially
the adjacency matrix of the input graph with rows concatenated.
\end{example}

One important property of queries is that they can be easily substituted
when one needs to check a formula on the resulting structure.
Given a $\tau \rightarrow \sigma$ query $q$, imagine we need to check
whether $q(A) \models \phi$ for some $\sigma$-formula $\phi$.
This can be done by replacing each relation $R_i$ and function $f_j$
in $\phi$ by the appropriate definition from $q$ and additionally
guarding all quantifiers to only quantify elements satisfying $\phi_0$,
the universe selection formula from $q$. Finally, we must add quantifiers
for each constant, defining it as the minimal tuple satisfying its
defining formula, and use these variables in place of the constant symbol.
In this way, we get a new
$\tau$-formula $\psi$ such that $q(A) \models \phi \iff A \models \psi$,
as formulated in the following lemma, equivalent to e.g.,
Proposition~3.5 of \cite{I99}.

\begin{lemma} \label{lem:invertq}
Let $q$ be a $\tau \to \sigma$ query and $\phi$ a $\sigma$-formula.
There exists a $\tau$-formula $q^{-1}(\phi)$ which satisfies,
for all $\tau$-structures $A$,
\[ q(A) \models \phi \iff A \models q^{-1}(\phi). \]
\end{lemma}


\subsection{Extensions of first-order logic}
\label{subsec:fo-extensions}

So far, we have focused only on first-order logic.  However,
first-order logic on finite structures is often too limited from
the computational perspective -- it cannot express many interesting
queries that are easy to compute.  In fact, over relational structures with 
additional numeric predicates, the first-order definable properties
correspond exactly to uniform AC$^0$ (cf.~\cite{I99}).  There are many
known correspondences between logics and complexity classes; we introduce
some of the relevant ones here.

To remove this limitation of FO, one extends it in various ways.
One option is to allow quantifiers over relation and function symbols,
resulting in \emph{second-order} logic.  We use SO to refer to
second-order logic restricted to relational signatures 
and SO$_\mathbb{R}$ to refer to second-order logic over $\mathbb{R}$-signatures.

Formally, formulas of SO (and SO$_\mathbb{R}$) are constructed in the same
way as formulas of FO (FO$_\mathbb{R}$) but with the added quantifiers
$\exists X \phi$ and $\forall X \phi$, where $X$ is a new relation symbol
(or function symbol, in SO$_\bbR$) of a fixed arity $r$, so $\phi$ can now contain atoms
(or terms, for function symbols) of the form $X(x_1, \dots, x_r)$.
In SO$_\bbR$ it is also possible to introduce second-order variables inside
terms using $\sup$, i.e., if $t$ is a SO$_\bbR$-term then so is $\sup_{F} t$,
where $F$ ranges over all functions $U^r \to \bbR$ for a given universe $U$.%
\footnote{We need to introduce separate variable binding for formulas and terms
because formulas and terms are distinct in our syntax. This can be avoided by using
a term-only syntax where formulas are a special case, as done in, e.g., \cite{KLLL15}.
But our syntax allows to trivially decide when to apply propositional solvers rather
than ones for the whole theory of the real field.  The particular choice of $\sup$
simplifies some later proofs, and it can be defined in SO$_\bbR$.  From a
complexity-theoretic perspective it complicates the relationship 
between certain fragments and complexity classes; this could be avoided but is not relevant to our purposes.}
Relational second-order variable $X$ with arity $r = 0$ is called a
\emph{bit-variable} since the atom $X$ can only be either true or false.
\emph{Existential} second-order logic is the fragment of second-order logic where all
second-order quantifiers are existential (while first-order quantifiers are not restricted),
and the $\sup$ operator is forbidden.

\begin{example}
Consider the following existential SO formula on graphs:
\begin{align*}
\exists & R, G, B\ \forall x,y\ \big( R(x)\lor G(x)\lor B(x)\big) \quad\land\\
  &\left( E(x,y)\rightarrow \lnot\left( (R(x)\land R(y))\lor (G(x)\land G(y))\lor (B(x)\land B(y))\right)\right)\,.
\end{align*}
This formula defines the well-known NP-complete problem of \emph{3-colorability} -- each vertex is colored red, green or blue and adjacent
vertices must have different colors.  Note that multicolored vertices are allowed, a multicolored vertex can be colored any of its
individual colors.
\end{example}

As this example indicates, second-order logic is very powerful; 
existential SO corresponds exactly to NP~\cite{Fagin74}.  This implies that coNP is captured
by universal SO, and that full SO captures the polynomial-time hierarchy.
The situation is similar for $\mathbb{R}$-structures, where existential SO$_\mathbb{R}$
captures NP$_\mathbb{R}$~\cite{GM96}, a class analogous to NP for computations with reals that
was defined by \cite{BSS}.

However, there is a large gap between uniform AC$^0$ and NP and it is desirable to
have logics corresponding to classes such as P.  This is done by extending first-order logic
with various operators.
For example, the \emph{transitive closure operator} allows
us to write formulas of the form $\TC[x_1,x_2.\varphi(x_1,x_2)](y_1,y_2)$.
This formula takes the transitive and reflexive closure of the (implicit)
relation defined by $\varphi(x_1,x_2)$ and evaluates it on $(y_1,y_2)$.
The \emph{least fixed-point operator} allows recursive definitions in formulas
of the form $\LFP[R(x_1,\ldots,x_k)=\varphi(R,x_1,\ldots,x_k)](y_1,\ldots,y_k)$,
where $R$ is a new relation symbol appearing only positively (i.e., under an even number
of negations) in the inner formula $\varphi$. The result of this operator is defined as
the least fixed-point of the operator $R(\ol{x}) \to \varphi(R, \ol{x})$. 
The \emph{functional fixed-point} is defined in a similar way over $\mathbb{R}$-structures,
see~\cite{GM96} for details.

\begin{example}
Consider the following formula on graphs augmented with constants $s,t$:
\[ TC[x,y.E(x,y)](s,t)\,.\]
This formula takes the transitive closure of the edge relation, and checks whether $(s,t)$ is in the result.  That is,
it defines the well-known NL-complete problem of $s,t$-reachability.
\end{example}

\begin{example}
Consider the following formula on weighted graphs augmented with constants $s,t$.
\[ TC[x,y.e(s,s)\le e(x,y)](s,t)\,.\]
This formula takes the transitive closure of the edge relation restricted to edges with weight at least
$e(s,s)$.  If we call this value $k$, then the formula defines the property of allowing a $k$-flow from $s$ to $t$ that is never split
over multiple edges.
\end{example}

Over relational structures\footnote{Recall that our structures are always ordered.
The existence of a logic capturing polynomial time on \emph{unordered} structures
is a major open question, cf.~\cite{Grohe08}.},  polynomial time is captured by least fixed-point logic 
(LFP) \cite{I86,V82}, and the same holds for P$_\mathbb{R}$ and functional 
fixed-point (FFP) \cite{GM96}.  Although LFP is presumably more expressive than 
transitive closure logic (TC), TC captures all problems solvable in non-deterministic 
logarithmic space (NL) on relational structures~\cite{I87}.

Of course, one can also consider extending SO with these operators; the
resulting logics capture well-known classes.  See~\cite{I99} 
for an overview of logics capturing other complexity classes.  All logics that we consider here are contained in SO$_{\bbR}$.

\subsection{Outlines}

Given a logic $\calL$, we refer to the set of $\calL$-formulas which may contain 
certain placeholders as \emph{$\calL$-formula outlines}.  Intuitively, an outline fixes
the structure of the formula but not the exact contents.

To be precise, we allow two kinds of placeholders.
First, atoms $a$ may be guarded by some \emph{Boolean guard}\footnote{We 
do not require that identical atoms share guards, that distinct atoms
have different guards, or that all atoms are guarded.} $G_i$.
Intuitively, the meaning of $G_ia$ is ``$a$ if $G_i$ and false otherwise''.
Boolean guards suffice for relational signatures.  In the case 
of $\bbR$-structures, formulas can contain real constants and it is desirable to
learn these constants automatically.  Thus, in addition to the Boolean guards,
we allow \emph{real placeholders} $w_i$.  Intuitively,
they represent real number constants which must be found.

More formally, we define \emph{$\calL$-formula outlines} as follows.
We fix a countable set of Boolean guards $\{G_1,\ldots\}$ and a countable
set of real placeholders $\{w_1,\ldots\}$.
Then, we define $\calL$-formula outlines exactly in the same recursive
way as $\calL$-formulas and number terms, with the following two additional rules.
First, for each outline $\varphi$ and Boolean guard $G_i$, $G_i\varphi$ is also
an outline. Second, each $w_i$ is also a number term outline.
Then, the set of formula outlines and number term outlines is built
in the same way as formulas and terms are built.

The Boolean guards are intended to mean ``$a$ occurs here'', and given an \emph{instantiation}
of the guards $I$, we can instantiate an $\calL$-formula outline $\psi$ to an
$\calL$-formula $\psi^I$ by replacing each $G_ia$ by $a$ if $G_i$ is true
in $I$, and by \texttt{false} otherwise. Similarly, an instantiation $I$ must
assign an algebraic real number $r_i$ to each $w_i$ to make it a number term.
We refer to queries containing $\calL$-formula outlines as \emph{$\calL$-query
outlines}. We omit $\calL$ when it is clear from context, and use \emph{outline}
to refer to both query and formula outlines.
Given an outline $o$, we write $\instt(o)$ for the set of formulas or queries
obtainable as instantiations of $o$.

Note that we do not allow the dimension of the query to be a placeholder,
that leads quickly to undecidability.  One could allow a finite upper-bound
on the dimension, but this can be simulated by a finite set of outlines.

Outlines are in some sense the logical equivalent to program sketches~\cite{sketching-conference}.
They have advantages including immediate upper bounds on the complexity of synthesized
formulas and clear normal forms.

\begin{example} \label{ex-triv-outline}
Consider a structure with a single binary relation symbol $E$.  An example outline
of a formula defining a binary relation with variables $x_1, x_2$ without equality is:
\[ G_1E(x_1, x_1) \lor G_2 E(x_1, x_2) \lor G_3 E(x_2, x_1) \lor G_4 E(x_2, x_2) \, .\]
Allowing equality in addition, a bit more complex example that we will use for learning actual reductions is an
outline defining a binary relation over a signature with constants $s,t$ and binary relation $E$:
\[
\vartheta_1 := \bigvee_{a,b\in\{s,t,x_1,x_2\}} 
           \left(G_{ab1}E(a,b)\lor G_{ab2}\lnot E(a,b)\lor
                 G_{ab3}a=b\lor G_{ab4}a\neq b\right) \, .
\]

We can use this formula outline as part of a query outline, e.g.
\[ q_1:= \left( k:=1,\ \phi_0:=\top,\ \phi_1:=\vartheta_1\right)\,.\]
\end{example}

\section{Learning Logical Queries}
\label{sec:learn}

In this section, we introduce our model of learning logical queries.
The model consists of a learner giving candidate queries or hypotheses
and a teacher (or verifier), which gives counter-examples or accepts the
query.  A learning task is characterized by a few parameters, first is
the target class $\calC$.

Let $\calC \subseteq \struc(\tau)\times\struc(\sigma)$ be a binary relation
on $\bbR$-structures, and define the domain of $\calC$ as
$\dom(\calC)=\{A\mid (A,B)\in\calC\mbox{ for some }B\}$.

In our definition of the teacher and the learner, we
distinguish between the relational part of a structure and
its real-valued functions. Recall that a structure $B$ is a
$\bbR$-modification of a structure $A$, $B \sim_\bbR A$, if it only
differs in the values of the real-valued functions, but keeps
the relational part intact. The restriction we put on the teacher
and the learner with respect to $\bbR$-modifications will become
clear later, when we discuss termination of the learning process.

\begin{definition}
\label{def:teacher}
A $\calC$-\emph{teacher} $t$ is a function 
\[t\colon (\tau\rightarrow\sigma)\mbox{-queries}   \to 
  (\dom(\calC)\times(\SO_{\bbR}(\tau)\times\SO_{\bbR}(\sigma))^*) \cup \{\top\}\,\]
that satisfies the following condition.
\[t(q)=
  \begin{cases}
   \top &\mbox{if }\left\{(A,q(A))\mid A\in\dom(\calC)\right\}\subseteq\calC,\\
   (A,(\phi^1,\psi^1),\dots\, &A\in\dom(\calC),(A,q(A))\not\in\calC,\models\vee_i\varphi_i,\text{ for each }i \le l: \\
   \phantom{(A,}\dots,(\phi^l,\psi^l)) &A'\sim_\bbR A, A' \models \phi_i \Rightarrow (B \models \psi_i \iff (A',B) \in \calC)\,
   \end{cases}
\]
\end{definition}

That is, a teacher accepts a query $q$ if for all $A \in \dom(\calC)$ we
have $(A, q(A)) \in \calC$, and otherwise replies with a counter-example $A$.
In addition to the counter-example,
the teacher provides a sequence of formulas $(\phi^i, \psi^i)$ that defines
the acceptable output on all $\bbR$-modifications of $A$.  Note the
condition $\models\vee_i\varphi_i$ requires
that at least one $\varphi_i$ holds on
every structure.
For relational signatures,
the definition can be simplified to returning $\top$ or $(A,\psi)$ since the only
$\bbR$-modification of a relational structure is the structure itself.

For $\bbR$-structures, the requirement that the teacher specifies the correct
behavior on all $\bbR$-modifications implies that not all classes $\calC$ have
a teacher. In fact the teacher can only specify Boolean combinations of polynomial
inequalities of real-valued functions from the structure. As a result, classes that
use real numbers for advanced computations (e.g., encoding undecidable problems
in the digits of the real numbers appearing there) do not have a teacher in this
model. We accept this limitation as our motivation for $\bbR$-structures is only
to allow easy access to basic computations with quantities.

Of course, in practice we generally restrict attention even more, to computable teachers
and ``reasonable'' classes~$\calC$. A natural extension would allow the teacher to return
multiple (at least one) counter-examples $A$ to an incorrect query, but we omit this
possibility for clarity of presentation.

\begin{example} \label{ex-red-teacher}
As a running example, we will trace the learning process for
a reduction from (directed) $s,t$-reachability to strong connectedness.
These properties can be defined in the following way.
\[\text{Reach}:=TC[x,y.E(x,y)](s,t) \quad \text{AllReach}:=
\forall x_1,x_2\ (TC[y,z.E(y,z)](x_1,x_2))\,,\]
and we write $G_\text{Reach}=\{A\mid A\models\text{Reach}\}$ for
the set of graphs satisfying $\text{Reach}$, and $G_\text{AllReach}$,
$G_{\lnot\text{Reach}}$, and $G_{\lnot\text{AllReach}}$
analogously.
The target class~$\calC$ is
\[\{(G_\text{Reach},G_\text{AllReach})\}
 \cup
  \{(G_{\lnot\text{Reach}}, G_{\lnot\text{AllReach}})\}\,.
  \]
The (general) teacher for such reductions is 
\[t(q)=
  \begin{cases}
   \top &\mbox{if }\forall A:A\models\text{Reach} \iff q(A)\models\text{AllReach} \\
   (A,(\text{Reach},\text{AllReach})) &\mbox{where } A\models \text{Reach}\land q(A)\models \lnot\text{AllReach}\\ 
   (A,(\lnot\text{Reach},\lnot\text{AllReach})) &\mbox{where } A\models\lnot\text{Reach}\land
   q(A)\models\text{AllReach}\,.
   \end{cases}
\]

Note that in general, it is uncomputable to check whether $A\models\text{Reach}\iff q(A)\models\text{AllReach}$ for all $A$.
\end{example}

Next, we define our learners.  To shorten the definition, let us
say that a query $q$ is \emph{consistent with} the series of examples
$(A_1, \ol{(\phi, \psi)}_1), \dots, (A_m, \ol{(\phi, \psi)}_m)$ iff,
for each $i \le m$ and each $A'_i \sim_\bbR A_i$, it holds that
if $A'_i \models \phi^j_i$ then $q(A'_i) \models \psi^j_i$.
Note that this is exactly the requirement from the teacher definition above.

\begin{definition}
\label{def:loglearner}
Let $\calH$ be a class of logical queries.
An $\calH$-\emph{learner} $L$ is a function that, given a sequence of examples
$e = (A_1, \ol{(\phi, \psi)}_1), \dots, (A_m, \ol{(\phi, \psi)}_m)$, satisfies
\[L((A_1, \ol{(\phi, \psi)}_1), \dots, (A_m, \ol{(\phi, \psi)}_m))=
  \begin{cases}
   h,    & h\in\calH,\text{ $h$ is consistent with $e$}, \\
   \bot, & \text{if no such } h \in \calH \text{ exists}\,.
  \end{cases}
\]
\end{definition}
A \emph{learning task} is specified by the pair $(\calC, \calH)$.

Note that our learners must always be consistent, and they return
$\bot$ iff there is no consistent query in the hypothesis space.
The logic used in the query is determined by~$\calH$. 

While Definitions~\ref{def:teacher} and~\ref{def:loglearner} have some appeal from
a theoretical perspective, in practice they have a number of disadvantages.
In  particular, we are interested in implementing our model and these definitions
may not be computable, and even when they are, they still require a new implementation
of the teacher and learner for each learning task.  
In the next subsections, we introduce a \emph{restricted} and \emph{uniform} learning model
based on formula outlines.  Restricted uniform learners and teachers have a number of practical
advantages -- for example, they are computable and the learning process is guaranteed to terminate.

\subsection{Restricted Uniform Learner}
\label{subsec:uniformlearner}

We begin by presenting \emph{restricted uniform learners}. These are
defined as \emph{outline learners}, which are the following.

\begin{definition}
An \emph{outline learner} is any $\calH$-learner such
that $\calH = \instt(q)$ for some query outline $q$.
\end{definition}

\begin{example}
To complement the teacher defined in Example \ref{ex-red-teacher}, we define an outline learner with $\calH=\instt(q_1)$
using the outline $q_1$ from Example~\ref{ex-triv-outline}.
\end{example}

The outline is uniform as it gives a compact representation of a hypothesis space, 
and can even enforce certain restrictions on the query. For example, outlines can
require a query to generate an \emph{extension}\footnote{An \emph{extension} of
a structure is formed by adding new predicates while leaving existing
predicates unchanged.} of the structure, which is useful when searching for
models to give explicit isomorphisms or satisfying solutions to SAT instances. 

Outlines are also quite restricted. For example,
for relational signatures, there are only finitely-many
instantiations of a query outline.  It is therefore
possible to simply try them all and return a suitable one.
However the following construction is preferable as
it allows to use modern efficient SMT solvers.

We start the construction with a technical lemma that simplifies
formula outlines for evaluation on structures of a fixed size.
Intuitively, we build a formula $\phi|_n$ that is in essence
a QBF equivalent to $\phi$ on structures of size $n$.
\begin{lemma} \label{lem:size-n-outline}
Let $\phi$ be a SO$_\bbR$ formula outline over a signature $\sigma$ and $n \in \bbN$.
For a structure $A$ of size $n$, let $A^c$ be an extension of $A$
by the constants $\{0, \dots, n-1\}$ with constant $i$ interpreted as element $i$.
There exists a SO$_\bbR$ formula outline $\phi|_n$ over the signature
$\sigma \cup \{0, \dots, n-1\}$ such that all items below hold.
\begin{enumerate}[(1)]
\item For all structures $A$ with universe of size $n$ and all instantiations $I$,
  \[ A \models \phi^I \iff A^c \models \phi|_n^I. \]
\item The guards and real placeholders in $\phi|_n$ are the same as in $\phi$.
\item The size of $\phi|_n$ is polynomial in $n$ and the size of $\phi$ (for a fixed $\sigma$).
\item There are no first-order quantifiers or sum terms in $\phi|_n$.
\item All relational second-order quantifiers in $\phi|_n$ are over bit-variables.
\end{enumerate}
\end{lemma}

\begin{proof}
The construction of $\phi|_n$ from $\phi$ proceeds inductively.
First order quantifiers are replaced by relational second-order bit variables.
For sum terms $\sum_x t$ we first introduce a second-order function variable
to define $t$ and then replace $\sum_x t$ by an explicit sum over possible $x$.
Finally, second order relational quantifiers are replaced by quantifying over all
bits that the actual relations can address on a structure of size $n$.

More formally, we set $\phi|_n = \phi$ for all atomic formulas $\phi$
and $t|_n = t$ for constant terms $t$ and real placeholders.
We then define $(\neg \phi)|_n = \neg(\phi|_n)$,
$(\phi \lor \psi)|_n = \phi|_n \lor \psi|_n$,
$(\phi \land \psi)|_n = \phi|_n \land \psi|_n$,
$(t = s)|_n = (t|_n = s|_n)$ and $(t < s)|_n = (t|_n < s|_n)$.
For guarded formulas set $(G_i \phi)|_n = G_i(\phi|_n)$. 
Analogously for terms: $(s+t)|_n = s|_n + t_n$, $(s-t)|_n = s|_n - t_n$,
$(s \cdot t)|_n = s|_n \cdot t_n$, $(s/t)|_n = s|_n / t_n$, and $\sgn(t)|_n = \sgn(t_n)$.
For first-order quantifiers, we define:
\[ (\exists x \, \phi)|_n = \exists X_0\ldots{}X_{n-1}\, (\exactlyone(X_i)\land\phi')\,.\]
Here, $\exactlyone(X_i)$ is the polynomial-size propositional formula stating that
exactly one of the $X_i$ is true, and $\phi'$ is formed from $\phi|_n$ by
replacing each atom containing $x$, e.g. $A(x,z)$, with
\[\left(\bigvee_{i \in \{0, \dots, n-1\}} (X_i \land A(i,z))\right)\,.\]
Function terms, e.g., $f(x, y)$, are replaced with
\[ \chi(X_0) f(0, y) + \chi(X_1) f(1, y) + \dots + \chi(X_{n-1}) f(n-1, y)\,.\]
For sum terms we define the $|_n$ operation as follows:
\[ \left(\sum_x t(x)\right)|_n =
\sup_{F_t}\left(\chi\left( \forall x \, t(x) = F_t(x) \right)|_n \cdot
           \sum_{i \in \{0, \dots, n-1\}} F_t(i)\right) \,.\]
Note that the summation on the right is an abbreviation and
not a sum term -- it is the linear-size explicit sum of the $F_t(i)$.

Finally a second-order relational quantifier is replaced by a series of quantifiers
over bit-variables.  For example, $\exists X \phi$
where $X$ has arity $2$, is replaced by 
$$\exists X_{00} X_{01} \dots X_{0n} X_{10} X_{11} X_{12} \dots X_{nn}\ \phi|_n.$$
Then, each atom $X(x, y)$ in $\phi|_n$ is replaced by
$\biglor_{i,j \in \{0,\dots,n-1\}}(x=i \land y = j \land X_{ij})$.
The properties listed in the lemma follow directly from this construction.
\end{proof}

The proof above was done directly for SO$_\bbR$ formulas, but it also works for 
operators such as TC and LFP. One can convert them in various ways,
for example they can be defined using second-order quantifiers.
However -- assuming the representation of formulas can handle
definitions efficiently -- it is more efficient to define
each stage of the induction in such operators in terms of
the previous stage and define the initial stage using the given formula.
For a structure of fixed size, such inductive definitions must halt after
logarithmically (for TC) or polynomially (for LFP) many steps, so the
converted formula size remains polynomial. The advantage of such
conversion over using a second-order definition is that 
we minimize the number of variables introduced.

Let us now use the above conversion to show how restricted uniform learners can be computed in practice.

\begin{lemma}
\label{lem:computablelearners}
Assume that we are given a $\tau\rightarrow\sigma$ SO$_\bbR$-query outline $q$ and a sequence
of $m$ examples $e = (A_1, \ol{(\phi, \psi)}_1), \dots, (A_m, \ol{(\phi, \psi)}_m)$ where each
$A_i$ is a finite $\tau$-structure, each $\phi^j_i$ is a SO$_{\bbR}(\tau)$-formula and
each $\psi^j_i$ is a SO$_{\bbR}(\sigma)$-formula.  We can compute a $q$-learner, i.e.,
a function $L^q$ satisfying:
\[
L^q((A_1, \ol{(\phi, \psi)}_1), \dots, (A_m, \ol{(\phi, \psi)}_m))=
\begin{cases}
   h,    & h\in\instt(q),\text{ $h$ is consistent with $e$}, \\
   \bot, & \text{if no such } h \in \calH \text{ exists}\,.
\end{cases}
\]
\end{lemma}

\begin{proof}
We will reduce the task of finding $h$ to the task of model-checking
a second-order $\{+, \cdot\}$ formula $\beta(q,e)$ using only second-order
bit-variables over the real field $(\bbR, +, \cdot)$. Note that first-order
quantifiers in this formula range over all real numbers, contrary to
all formulas used elsewhere in this paper. Since second-order bit-variables
can be simulated by first-order real-valued variables (e.g., by assuming
the bit $X$ is true iff the corresponding variable $x=0$), one can
convert $\beta(q,e)$ to a first-order formula over $(\bbR, +, \cdot)$.
Model-checking first-order formulas over the real field is known to
be computable~\cite{Tarski51} and efficient algorithms for
this problem exist~\cite{renegar98}. Importantly, SMT solvers can
be applied to check $\beta(q,e)$ directly. This is more efficient,
since the nature of bit-variables can be utilized in the solver.
We will also ensure that the size of $\beta(q,e)$ is polynomial
in the size of $q$ and $e$ and the maximum of the sizes of $A_i$.

To construct $\beta(q,e)$, recall that, by definition, $h$ is consistent
with $e$ iff for each $i \le m$ and each $A'_i \sim_\bbR A_i$, it holds
that if $A'_i \models \phi^j_i$ then $h(A'_i) \models \psi^j_i$.

Recall from Lemma~\ref{lem:invertq} that $q^I(A'_i) \models \psi^j_i$ is equivalent to
$A'_i \models (q^I)^{-1}(\psi^j_i)$. So the consistency condition above can be formulated as
\[ A'_i \models \bigland_j \left( \phi^j_i \rightarrow (q^I)^{-1}(\psi^j_i) \right). \]
Recall that the construction for $q^{-1}(\phi)$ was just a substitution of
the definitions from $q$ into $\phi$, so it also works when $q$ is a query outline
-- only then $q^{-1}(\phi)$ is a formula outline. Let therefore $\theta_i$ denote
the formula outline that defines consistency with the $i$-th example:
\[ \theta_i = \bigland_j \left( \phi^j_i \rightarrow q^{-1}(\psi^j_i) \right). \]
We can now equivalently reformulate our task as computing $L^q$ such that:
\[ L^q((A_1, \ol{(\phi, \psi)}_1), \dots, (A_m, \ol{(\phi, \psi)}_m))=
\begin{cases}
   q^I,  & A'_i \models \theta_i^I\text{ for all }i \le m, A'_i \sim_\bbR A_i, \\
   \bot, & \text{if no such instantiation } I \text{ exists}\,.
\end{cases}
\]

Now, since the size of each $A_i$ (and so each $A'_i \sim_\bbR A_i$) is known,
as these structures are given, we can use Lemma~\ref{lem:size-n-outline} and
instead of checking if $A'_i \models \theta_i$ we can check if $A'_i \models \theta_i|_{|A_i|}$.
Let $\theta'_i = \theta_i|_{|A_i|}$. Since $\theta_i$ has no free variables, the construction
from Lemma~\ref{lem:size-n-outline} provides a $\theta'_i$ that uses no first-order variables
at all, only constants appear in its atoms. Moreover, each relational atom, which now has
the form $R(i_1, \dots, i_r)$ for $i_1, \dots, i_r \in \{0, \dots, |A_i|\}$, has a known
truth value in $A'_i$ -- it's the same as $R(i_1, \dots, i_r)$ in $A_i$ since $A'_i \sim_\bbR A_i$.
So we can remove those, and we are left only with Boolean guards $G_i$ and number terms in which
all function terms have constants in their variables, i.e, are of the form $f_j(i_1, \dots, i_r)$.
Note that by Lemma~\ref{lem:size-n-outline} all relational second-order quantifiers in $\theta'_i$
are already over bit-variables, but we still need to handle second-order quantifiers over
real-valued function terms.

Let $\theta''_i$ be the formula $\theta'_i$ with each $f_j(i_1, \dots, i_r)$ replaced
by a new variable named $x^i_{f_j(i_1, \dots, i_r)}$. These variables now range over real
numbers. Also, replace each second-order quantifier over a real-valued function term,
e.g., $\forall f_j$, by a string of first-order quantifiers over the corresponding newly
introduced real variables for all occurrences containing $f_j$, e.g.,
$\forall x^i_{f_j(0)} x^i_{f_j(1)} x^i_{f_j(2)}$. Each term of the form $\sup_{f_j} t$
is similarly replaced by $\sup_{x^i_{f_j(0)} \dots x^i_{f_j(n)}} t$.
By Lemma~\ref{lem:size-n-outline} there are no sum terms in $\theta'_i$, and
we leave terms of the form $\chi(\phi)$ intact for the moment. We replace
each guarded formula $G_i\phi$ by $(G_i \land \phi)$ and we will treat the guards
$G_i$ as free bit-variables. We also treat real placeholders as variables and all
number-term functions ($+, -, \cdot$, etc.) as first-order functions. Note that
now the formula $\theta''_i$ is in the signature $\{+, -, \cdot, /, \sgn, <, r_i\}$
with additional $\sup$ and $\chi$ operators and uses only second-order bit-variables.
Note also that the condition that for all
$A'_i \sim_\bbR A_i$ holds $A'_i \models \theta_i^I$ is equivalent to
\[ (\bbR, +, -, \cdot, /, \sgn, <, r_i) \models
     \forall \ol{x^i_{f_j(i_1, \dots, i_r)}} \ (\theta''_i)^I, \]
where the quantifier prefix ranges over all newly introduced variables
$x^i_{f_j(i_1, \dots, i_r)}$.
Let us therefore construct the following $\{+, -, \cdot, /, \sgn, <, r_i\}$-formula:
\[ \alpha(q,e) = \bigland_{i=1,\dots,m} \forall \ol{x^i_{f_j(i_1, \dots, i_r)}} \theta''_i. \]
By the previous construction and the above equivalence we have
\[ L^q((A_1, \ol{(\phi, \psi)}_1), \dots, (A_m, \ol{(\phi, \psi)}_m))=
\begin{cases}
   q^I,  & (\bbR, +, -, \cdot, /, \sgn, <, r_i) \models \alpha(q,e)^I, \\
   \bot, & \text{if no such instantiation } I \text{ exists}\,.
\end{cases} \]

We will convert the formula $\alpha(q, e)$ constructed above to an equivalent
formula over the first-order theory of $(\bbR, +,\cdot)$. First, let us remove
the $\sup$ and $\chi$ operators. To that end, assume a $\sup_x r$ or $\chi(\psi)$
appears as a sub-term of $t$ in an expression $t = s$ (or $t < s$).
Let $t'$ be the term $t$ with the $\chi(\psi)$ or $\sup_x r$ sub-term
replaced by a new variable $z$. In case of $\chi(\psi)$, we replace $t = s$ by
\[ \exists z \left( (\psi \rightarrow z=1) \land
     (\neg \phi \rightarrow z=0) \land (t' = s) \right) \, . \]
In case of $\sup_x r(x)$, we replace the expression $t = s$ by
\[ \exists z \left( (\forall x (r(x) \leq z)) \land
     (\forall z' (\forall x (r(x) \leq z) \rightarrow z \leq z')
     \land (t' = s) \right) \, . \]

After recurrently applying the replacement procedure above, we are left
with a first-order formula over $-, /, \sgn, <$.
Recall that $-, /, \sgn, <$ and all algebraic real numbers are definable
in the real field using only $\cdot$ and $+$. We can thus write a
$\{+,\cdot\}$ formula $\alpha'(q,e)$ that is equivalent to $\alpha(q,e)$.
Let now $\{G_1, \dots, G_k\}$ be the set of all Boolean guards in $q$ and
$\{w_1, \dots, w_l\}$ the set of all real placeholders in this outline.
Note that these are all free variables in $\alpha'(q, e)$. So we set
\[ \beta(q, e) = \exists G_1 \dots G_k \, \exists w_1 \dots w_l \, \alpha'(q,e)\,.\]
By the above construction, $\beta(q,e)$ holds in the real field iff the assignment $I$
of the leading existential variables provides the $q^I$ we are searching for.
So we can use an SMT solver to solve $\beta(q,e)$ and set $L^q = \bot$ if it
answers \texttt{false} and otherwise get the leading assignment $I$ and set $L^q = q^I$.
\end{proof}

The formula $\beta(q,e)$ constructed in the proof above is polynomial in
the size of $q$, $e$ and $\max_i|A_i|$ if one allows to use definitions
inside formulas (which is allowed by all modern solvers).
Note also that if we work only on relational signatures, then there are
no real-valued variables or quantifiers in $\beta(q,e)$. Thus, it is
purely a  quantified Boolean formula (QBF) and there has been much recent
progress in efficient QBF solvers\footnote{See, e.g., the recent
QBF~\cite{qbf13,qbf14} competitions.}.  SAT solvers suffice for relational signatures
when the entire system is existential.  In many applications (see
Section~\ref{sec:appl} for examples) we do not require full SO$_\bbR$ and
these more limited formalisms can offer better performance.
An advantage of our approach is that in each application the complexity
is clear from syntax and so one can automatically choose to
use SAT or QBF solvers when possible.  

\subsection{Restricted Uniform Teacher}
 
\begin{definition}
\label{def:specification}
Let $\calC\subseteq \struc(\tau)\times\struc(\sigma)$ be a target class, \[P_{\calC}=\{(\Phi_1,\Psi_1),\ldots,(\Phi_p,\Psi_p)\}\]
be a finite set of formula pairs, and $n\in\bbN$.
We say that $(P_{\calC},n)$ is a \emph{specification} of $\calC$ if
all of the following items hold.
\begin{enumerate}[(1)]
 \item Each $\Phi_i$ is a second-order $\tau$-formula and $\Psi_i$ a second-order $\sigma$-formula.
 \item For all $A\in\dom(\calC)$, the size of $A$'s universe is $n$.
 \item $\dom(\calC) = \{A\mid A\in\struc^n(\tau),A\models\bigvee_i \Phi_i\}$.
 \item For $A\in\dom(\calC)$, $(A,B)\in\calC$ iff
  \[B\models \bigwedge_{\{i\mid A\models\Phi_i\}} \Psi_i\,.\]
\end{enumerate}
\end{definition}

Intuitively, $P_{\calC}$ is a conjunction of implications that defines the class $\calC$,
i.e, the acceptable behavior of the desired model $q$.  Given an input structure $A$,
if  $A\models\Phi_i$ then we require $q(A)\models\Psi_i$.  The restriction here to structures 
of size $n$ is needed to guarantee that the teacher in the following definition is computable.
One could similarly restrict attention to structures of size at most~$n$. To specify problems
without restricting the size, we say that $P_\calC$ above is a uniform \emph{unrestricted} specification
of $\calC$ if all above items except for (2) and the restriction to $\struc^n$ in (3) hold (we will re-visit those in subsection~\ref{subsec:iter}).

\begin{definition} \label{def:uniformteacher}
A \emph{uniform restricted teacher} is a $\calC$-teacher for a class $\calC$ that
has a specification $S=(\{(\Phi_1,\Psi_1),\ldots,(\Phi_p,\Psi_p)\},n)$.
\end{definition}

Note that a uniform restricted teacher is only concerned with structures of size $n$.
Also, recall that by Definition~\ref{def:teacher} it returns $\top$ iff
$\left\{(A,q(A))\mid A\in\dom(\calC)\right\}\subseteq\calC$. Given the specification $S$,
this condition is equivalent to saying that
\[ \text{for all } A\in\struc^n(\tau),\ i \le p, \quad A\models\Phi_i\implies q(A)\models\Psi_i\, .\]
Otherwise, the teacher is required to return a counter-example and a specification
of what should be done on $\bbR$-modifications of it: $(A,(\phi^1,\psi^1),\dots, (\phi^l,\psi^l))$.
A uniform teacher can always return the full specification $(\Phi_1,\Psi_1),\ldots,(\Phi_p,\Psi_p)$
instead of a list suited to the specific counter-example. Still, we leave the possibility to return
other formulas as it might improve the efficiency of learning.
Note also that there may be multiple choices of a counter-example $A$.
Any is acceptable, however the overall efficiency of learning may depend on the choice.

\begin{example} \label{ex-Tn}
The teacher from Example \ref{ex-red-teacher} is nearly
uniform -- all that remains is to fix $n$ as any finite value and
restrict the teacher to graphs of size $n$.
Then the teacher is uniform with specification
\[T_n:=\left(\{(\text{Reach},\text{AllReach}),(\lnot\text{Reach},\lnot\text{AllReach})\},n\right)\,.\]
\end{example}

We will now show that uniform restricted teachers are computable.
This is easy to prove for purely relational structures:
there are only finitely-many relational structures of size $n$
when the signature is fixed, and one can simply try them all.
In practice the following construction is preferable.

\begin{lemma}
\label{lem:computableteacher}
Let $S=(P_{\calC},n)$ be a specification of the class $\calC$.
There exists a computable uniform restricted $\calC$-teacher $t_S$.
\end{lemma}

\begin{proof}
The proof is similar to that for Lemma~\ref{lem:computablelearners},
and we will again construct a second-order $\{+, \cdot\}$ formula $\beta(q)$
using only second-order bit-variables and check it over the real field $(\bbR, +, \cdot)$.
Only this time the assignment of the leading existentially quantified variables
will provide the counter-example structure $A$.

By definition of a uniform restricted teacher, it returns $\top$ iff
for all structures $A$ of size $n$ and all $i$ it holds that
$A\models\Phi_i\implies q(A)\models\Psi_i$.
By Lemma~\ref{lem:invertq} we can rewrite $q(A) \models \Psi_i$
as $A \models q^{-1}(\Psi_i)$ so the whole condition becomes:
\[ A \models \bigland_i \left( \Phi_i \rightarrow q^{-1}(\Psi_i) \right)\, . \]
Since we are only concerned with structures of size $n$, let
\[ \theta = \bigland_i \left( \Phi_i \rightarrow q^{-1}(\Psi_i) \right)|_n\, . \]
Our task now is to find a structure $A$ of size $n$ that is a model
of $\neg \theta$, or return $\top$ if no such structure exists.

To this end, let again $\theta'$ be the formula $\theta$ with each $f_j(i_1, \dots, i_r)$
replaced by a new variable named $x^i_{f_j(i_1, \dots, i_r)}$ that ranges over reals.
Again, replace each second-order quantifier over a real-valued function term,
e.g., $\forall f_j$, by a string of first-order quantifiers over the corresponding
newly introduced real variables containing $f_j$, and similarly replace terms
of the form $\sup_{f_j} t$ by $\sup_{x^i_{f_j(0)} \dots x^i_{f_j(n)}} t$.
Also, treat all number-term functions ($+, -, \cdot$, etc.) as first-order functions.
Finally, replace each relational atom $R_k(i_1, \dots, i_r)$ by a new second-order
bit-variable $X_{R_k(i_1, \dots, i_r)}$. In this way, the constructed formula
$\theta'$ is in the signature $\{+, -, \cdot, /, \sgn, <, r_i\}$ and
uses only second-order bit-variables (and $\sup$ and $\chi$ operators).
The condition that for some structure
$A$ of size $n$ we have $A \models \neg\theta$ is equivalent to:
\[ (\bbR, +, -, \cdot, /, \sgn, <, r_i) \models
     \exists \ol{X_{R_k(i_1, \dots, i_r)}} \exists \ol{x^i_{f_j(i_1, \dots, i_r)}} \ \neg \theta' \, ; \]
the quantifier prefix ranges over all newly introduced bit-variables
$X_{R_k(i_1, \dots, i_r)}$ and real-valued variables $x^i_{f_j(i_1, \dots, i_r)}$.

Similarly as in the proof of Lemma~\ref{lem:computablelearners}, we use
the fact that $-, /, \sgn, <$ and all algebraic real numbers are definable
in the real field using only $\cdot$ and $+$, and that $\sup$ and $\chi$
can be defined as well. Substituting these definitions
into the formula on the right-hand side above yields the $\{+,\cdot\}$-formula
$\beta(q)$ which we then solve using an SMT solver. If there is no solution,
the teacher $t_S(q)$ returns $\top$. Otherwise, the SMT solver provides a witness for the outermost
quantified variables $X_{R_k(i_1, \dots, i_r)}$ and $x^i_{f_j(i_1, \dots, i_r)}$.
We construct the counter-example structure $A$ of size $n$ by putting
$(i_1, \dots, i_r) \in R_k$ if $X_{R_k(i_1, \dots, i_r)}$ is set to true,
and setting $f_j(i_1, \dots, i_r) = x^i_{f_j(i_1, \dots, i_r)}$ (if some
tuple is not quantified at all, we can set it to any number, e.g., $0$).
By the construction above, the reconstructed structure $A$ satisfies $\neg \theta$
and is thus a counter-example, as required. So we set $t_S(q) = (A, P_{\calC})$ in this case.
\end{proof}

While a general learning task is defined by $(\calC,\calH)$,
a \emph{uniform restricted} learning task is given by
a specification and outline: $((P_{\calC},n), q)$.

\begin{example}
To continue Example \ref{ex-red-teacher},
our (uniform restricted) reduction learning task for structures of
size $n$ is
$(T_n, q_1)$, where $T_n$ is from Example~\ref{ex-Tn} and $q_1$ from
Example~\ref{ex-triv-outline}.
We will see in Example \ref{ex-finally-finished-now} in Subsection~\ref{subsec:reduc} that one can learn a correct reduction for this example,
using the specification and techniques presented above.
\end{example}

\subsection{Termination of Uniform Restricted Learning}
\label{subsec:termination}

Let $L$ be a learner and $t$ a teacher.
We define the sequence $L^t_i$ of the interactions between $L$ and $t$ inductively
as follows. We set $L^t_0 := L()$, the hypothesis that $L$ returns on
an empty list of examples. If for some $i$ we get $L^t_i = \bot$ then
the sequence is finished -- there is no $h \in \calH$ that satisfies
the teacher. Else, let $E_i := t(L^t_i)$ be the answer of the teacher
to $L^t_i$. If $E_i = \top$ the sequence $L^t_i$ is finished, the last
hypothesis was accepted.  In the other case, set $L^t_{i+1} = L(E_0,\dots,E_i)$.

An outline learner is, essentially, a learner with a uniform hypothesis
space and a uniform restricted teacher is a uniform way of producing correct
counter-examples.  In the proofs for Lemmas~\ref{lem:computablelearners} and
\ref{lem:computableteacher} we saw how to convert the main conditions of
outline learners and uniform restricted teachers into model-checking on $(\bbR, +, \cdot)$.
This can be solved, and so we can guarantee an alternating sequence of consistent
hypotheses and counter-examples.

One concern is that we would like for the above sequence to \emph{terminate}, i.e. to
know after finite time whether there is an instantiation of the query outline that
satisfies the teacher.  While this is usually not achievable in the most general case,
it is always guaranteed for outline learners and uniform restricted teachers. 

\begin{theorem} \label{thm:finterm}
Let $S$ be a specification of a class $\calC$,
$t_S$ the uniform restricted teacher from Lemma~\ref{lem:computableteacher},
and $L$ an outline learner for some SO$_\bbR$-query outline $q$.
All of the following items hold.
\begin{enumerate}[(1)]
\item $L$ is a consistent and conservative $\instt(q)$-learner.
\item If $t_S(h) = \top$ for some $h \in \instt(q)$ then the sequence
  $L^{t_S}_i$ is finite and its last element $g$ satisfies $t_S(g) = \top$.
\item If there is no $h \in \instt(q)$ for which $t_S(h) = \top$ then
  the sequence $L^{t_S}_i$ is finite and its last element is $\bot$.
\end{enumerate}
\end{theorem}

\begin{proof}
The fact that $L$ is a consistent and conservative $\instt(q)$-learner,
as well as the correctness of the sequence $L^{t_S}_i$ follows directly
from the definitions. The only remaining thing is to show that the sequence
$L^{t_S}_i$ is finite. But note that there are only finitely-many $\tau$-structures
of size $n$ with different $\bbR$-modifications.
Given that each outline must hold on all previous counter-examples  (by Definition~\ref{def:loglearner}) and each next example
must be a counter-example (by Definition~\ref{def:teacher}), the sequence must terminate after finitely-many steps.
\end{proof}

For a given specification $S$ and query outline $q$, we will write
$L(S,q)$ for the last element of the sequence $L^{t_S}_i$ considered
above (which is well defined, since the sequence is finite).

Note that the proof above relies on the condition we imposed on teachers and
learners that all $\bbR$-modifications of a structure are handled simultaneously
in each step. It is easy to imagine a simpler learning definition, where in each
step the teacher only has to respond with a single $\bbR$-structure and a condition
applicable only to this structure, not all of its $\bbR$-modifications. The learner
would then construct a hypothesis correct only for these structures.

We did not use this simple model exactly because learning might not terminate.
Consider structures of size $1$, i.e., with only one element $0$, and only
a single real-valued function $f$. Imagine an outline $r > f(0)$ with a single
real placeholder $r$. Intuitively, the learner seeks a number $r$ greater than
the value $f(0)$ in the structure. For any finite sequence of examples $A_1, \dots, A_m$,
the learner will easily find such an $r = \max_{i \le m} f^{A_i}(0) + 1$. But then,
the teacher can respond with another example $A_{m+1}$ where $f(0) = r + 1$.
This would clearly result in an infinite learning sequence. Observe that the
condition on $\bbR$-modifications prevents this behavior: the learner will be
forced to answer $\bot$ already in the first step, as there is no $r$ bigger
than all numbers $f(0)$.

\subsection{Unrestricted Uniform Learning}
\label{subsec:iter}

In the previous two subsections we presented a restricted learning model
that can exploit the efficiency of SMT solvers. Let us now show how to
iterate the use of this model to get an unrestricted one. Theorems from
descriptive complexity will provide strong guarantees for this unrestricted
learning model. To this end, we need to say when a sequence of outlines
\emph{covers} a logic $\calL$.

\begin{definition}
Let $Q = \{q_1,q_2,\ldots\}$ be a sequence of query outlines in a signature $\sigma$
and let $\calL$ be a logic. We say that $Q$ \emph{covers} $\calL$ if for every
$\sigma$-outline $q$ from $\calL$, we have
$\instt(q)\subseteq\cup_{q_i\in Q}\instt(q_i)$\,.
\end{definition}

The definition above does not make any assumptions about computability of $Q$,
but in practice we will only use sequences $Q$ that are easily enumerable.
It is the advantage of our logic approach that such sequences can easily
be found by taking advantage of normal forms of formulas.

For example, consider a query outline $q^i_j$ that consists of $i$ disjunctions
of $j$ conjunctions of guarded atoms or guarded negated atoms from the signature $\sigma$.
The sequence $Q_0 = \{q^i_j \mid i,j \in \bbN\}$ consists of all outlines of
formulas in DNF. Since every quantifier-free formula can be converted to DNF,
we know that $Q_0$ covers all quantifier-free formulas. Similarly, we can
construct $Q_1$ by enumerating quantifier prefixes and putting them in front
of formulas from $Q_0$. Since every first-order formula has a prenex normal
form, we get that $Q_1$ covers FO. Putting a least fixed-point operator in
front of formulas from $Q_1$ gets us the set $Q_2$ that covers LFP, because
all LFP formulas have a normal form with just one LFP operator in the front.
In the next section, we will show a few sequences of query outlines that worked
well for practical applications.

Given a restricted specification $S$ and a sequence of query outlines
$Q = \{q_1, q_2, \dots\}$, we can run the restricted learning procedure
and compute first $L(S, q_1)$. If it is not $\bot$ then, by item~(2) of
Theorem~\ref{thm:finterm}, $L(S, q_1)$ is the solution for $S$. Otherwise,
since $L(S, q_1) = \bot$, we proceed to consider $q_2, q_3$, and so on.
If $Q$ covers a logic $\calL$ and $S$ has a solution in $\calL$, then it will
finally be found, since our procedure is complete. We write $L(S,Q) = L(S,q_i)$
for the smallest $i$ for which $L(S,q_i) \neq \bot$ and $L(S,Q) = \bot$ otherwise.

Consider now an \emph{unrestricted} specification $P_\calC$ of a class $\calC$,
and let $S_n = (P_\calC, n)$ be its restriction to structures of size $n$.
If $q \in Q$ is the first query from $Q$ that is an unrestricted solution
for $P_\calC$, then the sequence $L(S_i, Q)$ will stabilize on $q$ from
some $i$ on. We only get a guarantee that it is correct on structures of
size $i$ and below, but in practice it seems that queries that are
correct on moderately sized examples are usually correct in general
(where ``moderate'' depends on the complexity of the query). 
In addition, the following theorem shows that if a solution
exists in the complexity class we consider and we use a
suitable sequence of query outlines, then we will eventually converge to a correct solution.



\begin{theorem} \label{thm:guidance}
Let $Q$ be a sequence of query outlines covering FO (FO(TC), FO(LFP),
FFP, existential SO, existential SO$_\bbR$).
Assume $P_\calC$ is an unrestricted specification of a class $\calC$ and that
there exists a solution $f$ for $\calC$ (i.e. $(A, f(A)) \in \calC$ for all $A$)
that is in the complexity class uniform-AC$^0$ (NL, P, P$_\bbR$, NP, NP$_\bbR$).
Then $L((P_\calC, i), Q) \neq \bot$ for each $i$ and for some $k$ holds
\[ (A, L((P_\calC, l), Q)(A)) \in \calC \text{ for all } A \text{ and } l \geq k, \]
i.e., $L((P_\calC, l), Q)$ is a solution for $\calC$ for all $l \geq k$. 
\end{theorem}

\begin{proof}
By theorems from descriptive complexity cited in Subsection~\ref{subsec:fo-extensions},
if $\calC$ has a solution $f$ in uniform-AC$^0$ (NL, P, P$_\bbR$, NP, NP$_\bbR$),
then there exists a query $q$ in FO (FO(TC), FO(LFP), FFP, existential SO,
existential SO$_\bbR$) that is also a solution for $\calC$. Since $Q$ covers this
logic, we know that there exists a solution $q_i \in Q$. 

Let $S_n = (P_\calC, n)$. Since $q_i$ is a solution for the unrestricted
class $\calC$, it is also a solution for $S_n$. So $L(S_n, Q) \neq \bot$
since it will stop at $q_i$ at the latest. Also, for each $j < i$ such
that $q_j$ is not a solution for $\calC$, there exists a counter-example
for $q_j$ of size $n_j$. Let $k = \max_{j < i} n_j$. By Theorem~\ref{thm:finterm}
we will discard all false $q_j$ in $L(S_l, l)$ for all $l \geq k$, and thus
return a generally correct solution.
\end{proof}



The above theorem provides strong guarantees for our learning method:
if a solution exists, even in a broadly-defined complexity class such
as P or NP, then it will be found. The question remains whether this is
a practical method. In the next section, we examine various learning tasks,
look for reasonable teachers and outline sequences, and show that with
modern SAT, QBF and SMT solvers this method can indeed be practically applied.

\section{Applications}
\label{sec:appl}

As described above, the learning problem in our model consists of
the specification of the teacher and the outline for the learner.
For different learning tasks, it might be advantageous to choose different
outlines. For example, some tasks might only require a very simple Boolean
circuit to solve, while for other we might need the full power of polynomial
time programs with loops and intermediate definitions.
In this section we introduce a few parametrized classes of outlines with
increasing computational power. With each class of outlines, we present
a sample learning task that is well suited for this outline and discuss
how the task is solved in our model.

We start with outlines for very simple quantifier-free first-order formulas.
It turns out that even such basic outlines are useful: they are a good candidate
for finding reductions, as we discuss in the next section. After that, we move
to first-order outlines. These correspond to uniform $\text{AC}^0$ circuits and
we show that they can be used for learning patterns and rules on relational
structures, and even rules for board games. Next, we discuss how threshold
circuits can be encoded in our model. Threshold gates allow to build neural
networks and we show how such networks can be represented in
our model. Finally, we present outlines with fixed-point operators.
Such outlines can encode complicated polynomial-time programs and are hard to learn.
We present a few experiments where simple programs
with loops and definitions are successfully learnt in our model.

We focus here on examining how these sample applications can be achieved
in our model, in order to see that a variety of natural learning tasks
can be modeled.  
Subsections \ref{subsec:reduc} and \ref{subsec:formulas} contain
comparisons to existing, alternative approaches and show that
our methods are competitive.  Subsections
\ref{subsec:realprograms} and \ref{subsec:ptimeprograms} are primarily
intended to show the range of our approach and do not contain
exhaustive experimental results.

\subsection{Learning Quantifier-Free First-Order Formulas}
\label{subsec:reduc}

Complexity-theoretic reductions are an important tool to determine 
the relative hardness of computational problems and other applications
exist.  For example, SAT solvers are now commonly used as general NP
solvers and the necessary transformations are generally reductions.
This naturally leads to the question of (automatically) learning
and verifying reductions.

Learning reductions was first considered by Crouch \textit{et al.}~\cite{CIM10}, 
and we have also~\cite{JK13} implemented, benchmarked and evaluated a number 
of different approaches to the problem.

\paragraph{Problem}
In descriptive complexity, a reduction from the $\tau$-property
defined by $\varphi$ to the $\sigma$-property defined by $\psi$ 
is a ($\tau\rightarrow\sigma$)-query $q$ that satisfies 
\begin{equation}
\label{form:reddef}
A\models\varphi\iff q(A)\models\psi
\end{equation}
for all $\tau$-structures $A$.
Of course, reductions should have less computational power than the
complexity classes they are used in and descriptive complexity usually focuses
on weak classes of reductions, such as first-order reductions (i.e., first-order queries as reductions).
Here, we study quantifier-free first-order reductions, an even weaker class
that still suffices to capture important complexity classes.  While
polynomial time or logspace reductions are most common, such power
is usually not necessary for reductions and only causes additional
difficulties~\cite{Agrawal11,reduc_comp,Veith}.  
Here we introduce learning reductions in the context of our model,
see~\cite{JK13} for more details, other approaches and experimental comparisons.

In order to make finding quantifier-free reductions decidable,
we restrict attention to a fixed size $n$, i.e., we require
Formula~(\ref{form:reddef}) to hold only for structures of size at most $n$.
Assume that we are searching for a dimension-$k$ reduction from
the $\tau$-property defined by $\varphi$ to
the $\sigma$-property defined by $\psi$.

Let $P$ be the set of $\tau$-structures of size at most $n$ that satisfy $\phi$,
$Q$ be the set of $\sigma$-structures of size at most $n^k$ that satisfy $\psi$,
and $\overline{P}$ and $\overline{Q}$ be their complements up to the size bounds.
Our target class is $\calC=(P\times Q)\cup(\overline{P}\times\overline{Q})$ --
we want a query that maps positive instances to
positive instances and negative instances to negative instances.

\paragraph{Outline}
As an outline, we focus on reductions in which
all formulas are in DNF with $c$ conjunctions.  We fix $\varphi_0$ to
be always true and the dimension $k$ (so the new universe is the set
of $k$-tuples of elements of the old universe).  Finally, we have a number
of parameters determining the atomic formulas that may occur --
for example, whether to allow certain numeric predicates such as successor.

\paragraph{Teacher}
When the teacher receives a candidate hypothesis $q$, it checks
Formula~(\ref{form:reddef}), i.e. 
\[A\models\varphi\implies q(A)\models\psi \quad\land\quad A\not\models\varphi \implies q(A)\not\models\psi\]
for all structures $A$ of size at most $n$.
The teacher returns 
$(A,\psi)$ if $A\models\varphi$ and $(A,\lnot\psi)$ otherwise.

\paragraph{Results}

See~\cite{JK13} for an extended comparison of our approach using various SAT,
QBF, ASP and BDD packages, along with the earlier system developed by Crouch
\textit{et al.}~\cite{CIM10}.  Here we present a short summary of the results.

Learning quantifier-free reductions (with the restrictions described above) between
problems in NP\,$\cap$\,coNP is essentially
a $\Sigma_p^2$ problem.  Therefore it can be solved using a reasonable-sized
encoding and single call to either a QBF
solver or ASP solver supporting disjunctive programs.  We therefore compare our approach with modern
QBF and ASP solvers.

We refer to~\cite{JK13} for precise details and an extended experimental comparison
of various approaches to this problem.  In particular, there we present an
open-source implementation (DE\footnote{Available at \url{http://www-alg.ist.hokudai.ac.jp/~skip/de},
configured as \degms using GlueMiniSat 2.2.5 as solver.  Equivalent functionality is available as part of Toss: \url{http://toss.sf.net/}
and a visual interface is also available at \url{http://toss.sf.net/reduct.html}}) of our approach specialized
to reduction-finding.  QBF and ASP instance generation is done using \texttt{ReductionTest.native} (part of Toss,
see~\cite{JK13qbf} for details)\footnote{The input files used in this section are available at
\url{http://toss.sf.net/reductGen.html}}.

Table~\ref{tab:res} considers a set of 48
decision problems in NL (including e.g. directed and undirected reachabililty,
but also several simpler problems) and presents results for all $2304 = 48^2$ reduction-finding
problems constructed between these problems.
These 48 decision problems
are from the ReductionFinder implemented by Crouch \textit{et al.}~\cite{CIM10}
and therefore allow us to compare with ReductionFinder as well.  However, ReductionFinder
considers a slightly different class of reductions, generally resulting in somewhat
simpler instances.  A fair comparison is therefore difficult and ReductionFinder
is included for completeness.  The timeout was set to
120s in Table~\ref{tab:res}.

\begin{table}
\begin{center}
\begin{tabular}{c||r|r|r|r|r|r}
 $(c,n)$ & $(1,3)$ & $(2,3)$ & $(3,3)$
    & $(1,4)$ & $(2,4)$ & $(3,4)$ \\
  \hline
\degms         & 0  & 0   & 10   &
  0   & 5   & 103  \\
\rareqs        & 0   & 0  & 16  &
  19   & 65   & 204 \\
\depqbf        & 0   & 142  & 547  &
  16   & 297  & 711  \\
\gringo        & 40   & 393  & 590  &
  72   & 593  & 836  \\
\lparse        & 51   & 396  & 605  &
  75   & 635  & 850  \\
RedFind        & 1    & 152  & 396  &
  2    & 347  & 547  \\
\end{tabular}
\end{center}
\caption{Number of timeouts for tested
  reduction finding approaches, $k=1$.}
\label{tab:res}
\end{table}

We see that our approach, along with the QBF solver \rareqs, are the best
among these choices.  Interestingly, \rareqs is an expansion-based QBF
solver~\cite{JKSC12} and so in this case it essentially functions like
our approach by refining a series of hypotheses (abstractions) using
counter-examples.

While the parameters used in Table~\ref{tab:res} result in comparatively
easy instances, extended experiments with more difficult parameters give
similar results, cf.~\cite{JK13}.  In addition,
the teacher in Table~\ref{tab:res} considers only
counter-examples of size exactly $n$.  Allowing
counter-examples of size at most $n$ greatly improves
performance of our approach -- very small counter-examples result in easy sub-problems that
tell us a great deal about the space
of possible solutions.

\begin{example} \label{ex-finally-finished-now}
We complete our running Example~\ref{ex-red-teacher} here.
Recall that the task was to find a reduction between two
NL-complete problems: directed $s,t$-reachability
(given a directed graph with labeled vertices $s$ and $t$, determine
whether $t$ is reachable from $s$) and all-pairs reachability (determine if 
a directed graph is strongly connected). The problems were defined
by these 2 formulas.
\[\text{Reach}:=TC[x,y.E(x,y)](s,t) \quad \text{AllReach}:=
\forall x_1,x_2\ (TC[y,z.E(y,z)](x_1,x_2))\,.\]
We were searching through the space of outlines as described above,
and structures of increasing sizes.
Our system finds the following correct reduction for outline
$q_1$ as in Example~\ref{ex-triv-outline} and sizes $n\ge 3$:
\[\left(k:=1,\ \phi_0:=true,\ \phi_1:= x_1=s\lor x_2=t\lor E(x_2,x_1)\right)\,.\]

This reverses all edges in the original graph, adds
directed edges from $s$ to all vertices and also adds
directed edges to $t$ from all vertices.
A similar reduction exists without reversing the edges -- however
the above is our actual output.
\end{example}

\subsection{Learning First-Order Formulas}
\label{subsec:formulas}

Recently, a system was implemented \cite{K12} that represents board games
as relational structures and learns their rules from observing example play
videos. Fundamentally, the system works by computing minimal distinguishing
formulas for sets of structures, e.g., formulas satisfied by structures
representing winning positions and by none of the losing ones. We implement
the computation of distinguishing formulas in our framework and compare
the performance.

\paragraph{Problem}
Let $\calP$ and $\calN$ be finite sets of $\tau$-structures.
We want to learn a formula $\phi$ without free
variables such that $A \models \phi$ for all $A \in \calP$ and for
no $A \in \calN$. Unlike previous tasks, we want a minimal such formula,
not just an arbitrary one.



\paragraph{Outline and Teacher} 
The outlines in this case are not quantifier-free any more,
but they are built by adding quantifier prefixes to quantifer-free
outlines similar to the ones used above. In this case, we start
with a CNF formula with $c$ clauses and $k$ additional variables.
Then, we quantify the $k$ additional variables existentially.
The final outline is then a disjunction of $l$ such quantified CNF formulas,
each with $k$ added variables and at most $c$ clauses.
Moreover, to find minimal formulas we iterate through $k$, for each
$k$ we range $l$ from $1$ to $k+1$, and $c$ as well. The teacher is simple:
given a formula $\phi$ it checks if $A \models \phi$ for all $A \in \calP$,
and if not, it returns $(A,\text{true})$ for some $A \not\models \phi$.
Then, it checks if $A \not\models \phi$ for all $A \in \calN$ and
returns $(A,\text{false})$\footnote{Here, ``true'' and ``false'' are satisfied
by encoded Boolean structures.} if this is not the case for some
$A \models \phi$.

\paragraph{Results}
We substituted our SAT-based learner for the procedure for computing
distinguishing formulas used in \cite{K12}. To replicate the experiments,
we used the most recent revision of Toss\footnote{Revision 1935 on Sourceforge,
compiled with OCaml 4.02.1.} and ran each experiment 3 times on a 4Ghz Intel i7-4790K
processor. Since the variance in time was negligible, we only report the mean
running time.\footnote{To replicate, after getting Toss, do
\texttt{make Learn/LearnGameTest.native} and then \texttt{LearnGameTest.native -dir Learn/examples/ -f Breakthrough001}, or \texttt{Connect4001}, etc. for the original results,
and with \texttt{-s} to use our SAT solver technique.} Comparing the results,
the SAT-based approach appears to offer significantly better performance, even though
it is more general and competes against a system hand-crafted specifically for this problem.
\begin{center}
\begin{tabular}{c||c|c|c|c}
   & Breakthrough & Connect4 & Gomoku & Pawn-Whopping \\
\hline
Original system  & 39s & 14s & 4s & 473s \\
SAT-based system & 2s  & 5s  & 2s & 130s \\
\end{tabular}
\end{center}

We use the same example plays for both systems -- these examples were chosen
by hand for the original system \cite{K12}. But our SAT-based approach
searches (faster) for formulas in more expressive logics, beyond reach for
the original system.  For this reason, the resulting formulas are not
always correct -- they are for Breakthrough and Gomoku, but not for
Connect4 and Pawn-Whopping. It would be easy to overcome this
by adding examples or changing the outline to match the more restrictive
logics used by the original system.

\subsection{Learning Threshold Circuits (Formulas with Reals)}
\label{subsec:realprograms}

Let us now show how neural networks can be represented in our model.
We will focus on deep \emph{convolutional networks}, but the ideas
generalize to other models easily.

Convolutional neural networks \emph{share weights}  using
\emph{sliding windows} over the input vector and often alternate
such shared-weights-layers with max-pooling layers which just compute
maximum over a window to reduce the number of neurons and, more
importantly, to capture different scales, e.g. in image recognition.
In the convolutional layer, each neuron is a so called \emph{rectified
linear unit}. This means that each neuron multiplies its inputs by
the respective weights, adds the result, subtracts another weight,
and sends to its output the maximum of this result and 0.

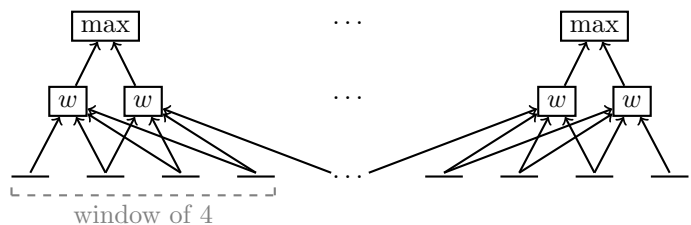
\begin{center}
\begin{tikzpicture}[scale=0.5] \small

\draw[thick] (-9, -2) -- (-8, -2) ;
\draw[thick] (-7, -2) -- (-6, -2) ;
\draw[thick] (-5, -2) -- (-4, -2) ;
\draw[thick] (-3, -2) -- (-2, -2) ;
\draw[thick] (3, -2) -- (2, -2) ;
\draw[thick] (5, -2) -- (4, -2) ;
\draw[thick] (7, -2) -- (6, -2) ;
\draw[thick] (9, -2) -- (8, -2) ;
\node at (0, -2) {\dots} ;

\draw[thick, dashed, color=gray] (-9, -2.5) -- (-2, -2.5) ;
\draw[thick, dashed, color=gray] (-9, -2.5) -- (-9, -2.1) ;
\draw[thick, dashed, color=gray] (-2, -2.5) -- (-2, -2.1) ;
\node[color=gray] at (-5.5, -3) {window of $4$} ;

\node[draw,thick] (a) at (-7.5, 0) {$w$} ;
\node[draw,thick] (b) at (-5.5, 0) {$w$} ;
\node[draw,thick] (c) at (5.5, 0) {$w$} ;
\node[draw,thick] (d) at (7.5, 0) {$w$} ;
\node at (0, 0.1) {\dots} ;

\draw[thick,->] (-8.5, -1.9) -- (a) ;
\draw[thick,->] (-6.5, -1.9) -- (a) ;
\draw[thick,->] (-4.5, -1.9) -- (a) ;
\draw[thick,->] (-2.5, -1.9) -- (a) ;

\draw[thick,->] (-6.5, -1.9) -- (b) ;
\draw[thick,->] (-4.5, -1.9) -- (b) ;
\draw[thick,->] (-2.5, -1.9) -- (b) ;
\draw[thick,->] (-0.5, -1.9) -- (b) ;

\draw[thick,->] (6.5, -1.9) -- (c) ;
\draw[thick,->] (4.5, -1.9) -- (c) ;
\draw[thick,->] (2.5, -1.9) -- (c) ;
\draw[thick,->] (0.5, -1.9) -- (c) ;

\draw[thick,->] (8.5, -1.9) -- (d) ;
\draw[thick,->] (6.5, -1.9) -- (d) ;
\draw[thick,->] (4.5, -1.9) -- (d) ;
\draw[thick,->] (2.5, -1.9) -- (d) ;

\node[draw,thick] (m1) at (-6.5, 2) {$\max$} ;
\node[draw,thick] (m2) at (6.5, 2) {$\max$} ;
\node at (0, 2.1) {\dots} ;

\draw[thick,->] (a) -- (m1) ;
\draw[thick,->] (b) -- (m1) ;
\draw[thick,->] (c) -- (m2) ;
\draw[thick,->] (d) -- (m2) ;
\end{tikzpicture}
\end{center}

The figure above sketches a convolutional network. We wrote $w$ in each neuron
to emphasise that \emph{the same weight vector} is used in all neurons
in this layer. For vision applications, the sliding window and max-pooling
window can be 2D, e.g. $3\times 3$. For classification, there is often a
fully-connected layer before the output layer. All these can be encoded in
our formalism in an analogous way, so we focus on 1D convolutional networks.

The input to a 1D network is a $n$-bit vector of real numbers. We represent
it as a relational structure over $U = \{0,\dots,n-1\}$ with a unary
function $f_0$, such that $f_0(i)$ is the $i$th component of the input vector.

To keep our reduction simple, let us assume our network has two convolutional
layers with a window of size $3$ and two maps each, with one
max-pooling layer in between. The outline we provide for this network generalizes in
an easy way to other network architectures.

By definition, the output of the $i$th neuron of the first map of
the first convolutional layer, which we denote $f^1_1(i)$, is given by
\[ f^1_1(i) = \max(0, w^1_1 \cdot f_0(i) + w^1_2 \cdot f_0(i+1) + w^1_3 \cdot f_0(i+2) - w^1_0). \]
In our formalism, we can write $\max(0, t)$ as $t \cdot \chi(t > 0)$,
so let us use $\max$ as a more readable shorthand. We can formalize
the above as follows:
\[ f^1_1(x) = \max\left(0, \sum_{y, z : \suc(x, y) \wedge \suc(y, z)} \!\!\!\!
     w^1_1 \cdot f_0(x) + w^1_2 \cdot f_0(y) + w^1_3 \cdot f_0(z) - w^1_0\right). \]
Note that $w^1_0$, $w^1_1$, $w^1_2$, and $w^1_3$ are now the real placeholders of our outline.
We define $f^2_1$ analogously with $w^2_i$ (i = 0,\dots,3) and get the 2 maps of
the first convolutional layer complete in this way.

However, we must still construct the max-pooling and following layer, where there are
fewer neurons due to the scaling effect of max-pooling.  This is done
by skipping every other element. Let us first define the max-pooling outputs:
\[ f^i_{\text{m}}(x) = \max\left(f^i_1(x), \sum_{y:\suc(x,y)} f^i_1(y)\right) \]
for $i=1,2$. Let $\nsuc(x,y) := \exists z (\suc(x,z) \wedge \suc(z,y))$.
We create the second convolutional layer by skipping over every other element.
\begin{align*}
\hspace{-0.35in}f^i_2(x) = \max\Bigg(0, \sum_{y, z : \nsuc(x, y) \wedge \nsuc(y, z)} \!\!\!\! &
  w^2_{1} \cdot f^1_{\text{m}}(x) + w^2_{2} \cdot f^1_{\text{m}}(y) + w^2_{3} \cdot f^1_{\text{m}}(z) + \\
& w^2_{4} \cdot f^2_{\text{m}}(x) + w^2_{5} \cdot f^2_{\text{m}}(y) + w^2_{6} \cdot f^2_{\text{m}}(z) - w_0\Bigg).
\end{align*}
Now, we can trivially make a query that selects even elements from
our input structure and uses $f^1_2$ and $f^2_2$ as functions.
It represents exactly the output of the second convolutional layer
in a network.

As we have seen, neural networks fit nicely into our model, in fact
they correspond exactly to a specific syntactic class of formulas. But we
observed a problem when experimenting with such encodings using the method
presented in the proofs above: SMT solvers are generally not efficient
when dealing with neural network problems of this kind. Similar
issues have been
previously reported \cite{PT12} and we are not aware of a fully satisfying
solution at this time. Still, the approach we present can be adapted
to make use of partial procedures, such as stochastic gradient descent,
as part of the learner process.

\subsection{Learning Polynomial-Time Programs}
\label{subsec:ptimeprograms}

In this section, we consider synthesizing programs for a given logical
specification. Itzhaky \textit{et al.}~\cite{IGIS10} considered a similar problem, however they
focused on synthesizing formulas in more specialized logics.

\paragraph{Problem}
In program synthesis, we are given a specification and hope to find
an efficient program satisfying it.  For us, a specification is a way
to verify whether the output $q(A)$ is accepted for $A$.  There are
two major variations -- either the output for each structure is unique
(as in our example here), or there is a set of acceptable outputs (e.g.,
when finding some satisfying solution for SAT instances).

In our example here, we have a query $s$ in an expressive logic ($\SO$)
and wish to find an equivalent query in a less-expressive logic.  In
particular, we consider the problem of identifying winning regions in
finite games -- i.e., directed graphs with a predicate $V_0(x)$ meaning that
vertex $x$ belongs to Player~0.  Decidability requires restricting the size of
the games to $n$, and we set $\calC$ to be the set of pairs $(A,B)$ such that
$A$ is a finite game of size at most $n$ and $B$ is the extension of $A$ with
the winning region identified in a new monadic predicate $W$.

\paragraph{Outline}
Here, we re-use the outlines introduced for learning games, but
with an added extension to least fixed-point formulas for added
expressive power.  We focus on least fixed-point formulas, with a single
LFP operator that is outermost\footnote{This
is a normal-form for fixed-point logics, although the arity of
the fixed point may increase when we convert to it (see Corollary 4.11 in~\cite{I99}).}.
We fix the arity~$a$ of the fixed-point predicate, and assume, as before, that
the inner formula is a disjunction of $l$ quantified CNF formulas
with $k$ variables and $c$ clauses each.  We use such an outline for exactly one
selected relation $W$ in the query, all others are set to identity.

\paragraph{Teacher}
Assume that we have a $\SO$ query $s$ that produces the desired extension with
the winning region identified.
Given a hypothesis $q$, the teacher can guess a game $A$
of size at most $n$ such that in $q(A)$ the new relation $W(x)$ is not
equivalent to the region in $s(A)$. Let $a_1, \ldots, a_k$ be the winning
positions in $s(A)$.  The teacher then returns the pair
$\big(A, \forall x\ (W(x) \liff \biglor_i (x = a_i))\big)$.

\paragraph{Example}
Consider the case of identifying winning regions in finite reachability games.
When the current vertex belongs to a player, that player chooses
an outgoing edge and moves to a connected vertex.  Player~1 loses
if the play reaches a vertex that belongs to her and has no outgoing edge.
Similarly, Player~0 loses if the play reaches a position where he must but
cannot move, but also if the play becomes a cycle and goes on forever. 
The goal is to identify the vertices from which Player~0 has a winning strategy.
That is, we want a formula $\phi(x)$ which holds exactly on
the vertices for which Player~0 has a winning strategy.

Reachability games are \emph{positional}, i.e., it suffices to
consider strategies that depend only on the current position and
not on the history of the game. Therefore, a strategy of Player~1 can be
defined as a binary relation $S_1$ that is a subset of the edges and that,
for vertices of Player~0, contains all successors of the vertex as well.
Then, $x$ is winning for Player~1 (by $S_1$) if all vertices reachable from $x$ 
by $S_1$ either belong to Player~0 or have an $S_1$-successor. This is easily
expressible using the TC operator and guessing the strategy leads to a 
second-order formula $\phi_0(x)$ which holds exactly if $x$ is winning for
Player~0.

In reachability games, the following LFP formula defines the winning region
for Player~0:
\begin{align*}
\LFP\Big[W(x)= \Big\{&\left(\lnot \exists x_1 : (\lnot W(x_1)\land E(x,x_1))\land \lnot V_0(x)\right)\quad \lor \\
 &\,\hphantom{(\lnot}\exists x_1 : (W(x_1)\land E(x,x_1)\land V_0(x))\Big\}\Big](x)\,.
\end{align*}

The LFP operator recursively defines $W(x)$, starting with the empty set
and adding tuples that satisfy the formula on the right until a fixed-point is
reached. 
 Therefore, this formula says that a position $x$ is winning for Player~0 if\\
(a) it is the opponent's move and all outgoing edges go to positions we win;
or (b) it is Player~0's move and there is an edge to a winning position.

Recall that the fixed-point predicate is initially empty.  Therefore, the
winning positions after one iteration are the positions belonging to the
opponent with no outgoing edges.  Then, the winning region grows gradually
until it is the \emph{attractor} of those positions -- which is correct.  
An equivalent, slightly longer formula is found by our program in less than 
a minute for $n\ge3$.\footnote{To replicate, after getting Toss, do
\texttt{make Learn/LfpTest.native} and then run it.}
\paragraph{Further work}
The LFP formula for reachability games can be written by hand, but our
motivation for presenting this example is the hope to compute polynomial-time
solvers for other games. In particular, weak parity games and parity games are
also positional, so it is trivial to write a $\SO$ formula defining the
winning region (as we did for reachability games). But the polynomial-time 
program
for solving weak parity games is complicated, and the existence
of a polynomial-time solver for full parity games (which is equivalent to
the existence of an LFP formula) is a long-standing open problem.

Our implementation can also search for other programs, e.g.,
for graph isomorphism, graph coloring or SAT.
There are classes where these problems are in polynomial-time\footnote{E.g.,
isomorphism of planar graphs, bounded-degree graphs, and graphs excluding
a minor; $k$-SAT and $k$-coloring are NL-complete
for $k=2$ and NP-complete for $k\ge 3$.}.

\section{Conclusions and Future Work}
\label{sec:concl}

Above, we introduced our machine learning approach and its implementation
using modern SMT solvers. We prove that our approach comes with strong
theoretical guarantees: as long as a model exists in a given complexity
class (e.g., NL, P, NP), it will be found.
Thanks to the efficiency of modern SAT and QBF solvers, our
general procedure outperforms specialized approaches both in learning
reductions~\cite{CIM10} and in learning from examples~\cite{K12}.
We consider these early results promising and encourage further
experimentation with our freely available implementation\footnote{Available from
\url{http://toss.sf.net}, see~\cite{JK13} for learning instructions.}.

There are many questions we leave unanswered.
For example, there is a large variance in runtime depending on
the precise series of counter-examples given by the teacher. We would
like to know how to choose ``good'' counter-examples and whether
randomness~\cite{Z06} can help. We ask what outlines are ``good'',
how to choose them to find the desired programs quickly and to make them 
readable.
We are interested in re-using sub-formulas found for one problem to speed up
learning in another one, a form of knowledge transfer.

Finally, efficiency of our approach in quantitative settings has been disappointing. 
It seems that at present SMT solvers are not suited for tasks such
as verification or synthesis of neural networks. But SMT
solvers are evolving and hopefully will improve in this regard.
One could also use stochastic gradient descent directly in our learner.
How to utilize such incomplete procedures efficiently remains an open question,
but we hope that this work will motivate further studies on the boundary
between SMT solvers and machine learning.

\paragraph{Acknowledgements}
We appreciate the encouragement and support of
Neil Immerman and Thomas Zeugmann, without
whom this paper would not exist.

\bibliography{synthesis}
\bibliographystyle{theapa}



\end{document}